\documentclass{article}

\usepackage{graphicx}
\usepackage{amsfonts}
\usepackage{amsthm}
\newtheorem{theorem}{Theorem}

\newtheorem{remark}{Remark}

\usepackage{amsmath}
\usepackage{amssymb}
\usepackage{mathrsfs} 
\usepackage[colorlinks=true, linkcolor=blue, urlcolor=blue, citecolor=blue]{hyperref}

\usepackage{natbib}
\usepackage{epstopdf}

\usepackage[margin=1in]{geometry}

\usepackage[utf8]{inputenc} 
\usepackage{hyperref}       
\usepackage{url}            
\usepackage{booktabs}       
\usepackage{amsfonts}       
\usepackage{nicefrac}       
\usepackage{microtype}      
\usepackage{xcolor}         

\usepackage[linesnumbered,ruled,vlined]{algorithm2e}
\DeclareMathOperator*{\argmin}{arg\,min}
\DeclareMathOperator*{\argmax}{arg\,max}

\title{Online Conformal Model Selection for Nonstationary Time Series}

\author{Shibo Li and Yao Zheng
	\\ \textit{University of Connecticut}}
\date{}

\begin{document}
	
\setlength{\parindent}{16pt}
\maketitle


\begin{abstract}
	This paper introduces the \textit{MPS (Model Prediction Set)}, a novel framework for online model selection for nonstationary time series. Classical model selection methods, such as information criteria and cross-validation, rely heavily on the stationarity assumption and often fail in dynamic environments which undergo gradual or abrupt changes over time. Yet real-world data are rarely stationary, and model selection under nonstationarity remains a largely open problem. To tackle this challenge, we combine conformal inference with model confidence sets to develop a procedure that adaptively selects models best suited to the evolving dynamics at any given time. Concretely, the MPS updates in real time a confidence set of candidate models that covers the best model for the next time period with a specified long-run probability, while adapting to nonstationarity of unknown forms. Through simulations and real-world data analysis, we demonstrate that MPS reliably and efficiently identifies optimal models under nonstationarity, an essential capability lacking in offline methods. Moreover, MPS frequently produces high-quality sets with small cardinality, whose evolution offers deeper insights into changing dynamics.  As a generic framework, MPS accommodates any data-generating process, data structure, model class, training method, and evaluation metric, making it broadly applicable across diverse problem settings.
\end{abstract}

\section{Introduction}

The popularity of online time series modeling has surged due to the growing need for real-time, adaptive forecasting \citep{Liu2016, Zhao2022, Bhatnagar2023, Wang2024}. With the continuous influx of data, forecasters and decision-makers must now process and update models instantaneously as data streams in. Yet, real-world time series  are rarely stationary due to unforeseen events, structural changes, or evolving dependence structures \citep{Dahlhaus2012, Aue2013, Ditzler2015, Baker2020}. The shift toward real-time modeling brings unique challenges particularly in model selection \citep{Kley2019,Wang2022}.


Classical model selection methods for time series, such as information criteria,  cross-validation, and likelihood-based approaches, rely heavily on the stationarity assumption \citep{McQuarrie1998, box2015, Hyndman2021}. However, offline methods fall short in online settings where the true model may evolve over time. Moreover, in a changing environment, model selection is inherently associated with a level of uncertainty. This may arise from similarly competitive models, which is increasingly common with the advancement of modern forecasting techniques; thus, the optimal model may not be a single model, but rather a groups of models.
Uncertainty may also stem from limitations in available data. For example, during an ongoing global pandemic or a sudden policy change, limited data can make it difficult to determine whether the resulting economic disruption should be treated as a permanent structural shift---warranting a change in the model---or merely as a temporary outlier \citep{Stock2025}. Prompted by the complexities of real-time data and modern modeling techniques, this paper addresses a key yet understudied question: in online settings, how can we perform model selection that adapts to unknown forms of nonstationarity in time series, while also accounting for the uncertainty inherent in the selection process?

\paragraph{Contributions} To tackle the challenge of adaptive model selection in online nonstationary settings, we introduce the  \textit{Model Prediction Set (MPS)}. This procedure updates a confidence set that, among a collection of candidate models $\mathscr{M}$, covers the best model for the next time period with a specified level of confidence over the long run, while adapting to changing dynamics in real time. 

\begin{figure}[t]
	\centering
	\includegraphics[width=0.6\textwidth]{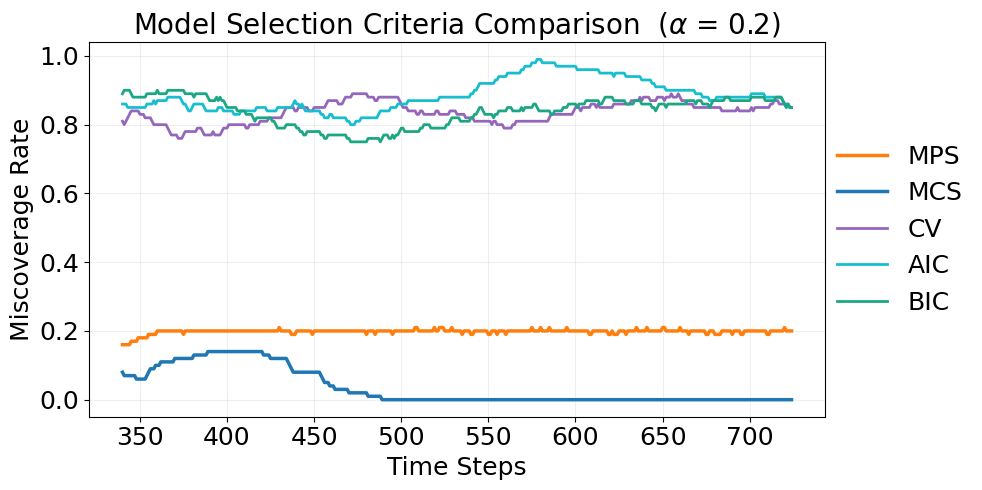}
	\caption{Miscoverage rates (i.e., the proportion of times the best model $\mathcal{M}_{t+1}$ is not included in $\mathcal{C}_t$ up to time $t$ evaluated over a moving window of size 100) of several model selection methods applied to forecasting with the ETTh1 dataset (see Section \ref{sec:data}). We compare: (i) offline single-model selection approaches (AIC,  BIC, and CV based on minimizing forecast error over a hold-out set); (ii) offline MCS; and (iii) the proposed MPS. The best model is defined using one-step-ahead forecast error (forecasting time $t+1$ based on data up to time $t$) as the evaluation metric. As shown, only  MPS achieves accurate control of miscoverage close to the nominal level $0.2$. All single-model selection methods perform poorly, and MCS exhibits extremely low miscoverage by producing trivial sets that ultimately include all candidate models.  See Appendix \ref{sec:a1} for experiment details and Section \ref{sec:data} for a more detailed comparison of MPS and MCS. }
	\label{fig:ModelSelction}
\end{figure}

Specifically, let $\mathcal{M}_t\in\mathscr{M}$ be the \textit{best} model at each time $t$, where the optimality is defined based on a user-chosen model evaluation metric. Clearly, predicting the best model $\mathcal{M}_{t+1}$ for the subsequent period, given information up to the current time $t$, is extremely challenging in a nonstationary setting. See Figure \ref{fig:ModelSelction} for an illustration of the poor performance of offline single-model selection approaches, including Akaike information criterion (AIC),  Bayesian information criterion (BIC), and cross-validation (CV), in selecting the best forecasting model for the next time point.  Instead, a more realistic and practical objective is to construct a confidence set $\mathcal{C}_{t}$ of competing models, based on history up to the current time $t$, that is guaranteed to include $\mathcal{M}_{t+1}$ with a certain level of confidence. Note that since   $\mathcal{M}_{t+1}$ is not revealed until time $t+1$, the confidence set $\mathcal{C}_{t}$ is indeed a \textit{prediction set}, similar to the concept of prediction intervals. Formally, 
given $\mathscr{M}$, as any time series data are continuously collected, our method produces a sequence of sets $\{\mathcal{C}_{t}\}$ such that
\begin{equation}\label{eq:target}
	\lim_{T\rightarrow\infty} \frac{1}{T}\sum_{t=1}^{T}\mathbf{1}\{\mathcal{M}_{t+1} \notin \mathcal{C}_{t}(1-\alpha_t)\} \leq \bar{\alpha},
\end{equation}
where $ \bar{\alpha}\in(0,1)$ is a pre-specified target miscoverage rate,  and $\mathcal{C}_{t}=\mathcal{C}_{t}(1-\alpha_t)\subset\mathscr{M}$ is the MPS constructed at time $t$, with $\alpha_t\in(0,1)$ being the nominal miscoverage rate which is adaptively calibrated based on information available up to time $t$.

The MPS is highly generic. It does not rely on the identification of a true model, nor does it require the time series to be stationary. The definition of  the \textit{best} model is also flexible: the process of the best models $\{\mathcal{M}_{t}\}$ can be defined based on any user-chosen empirical model evaluation metric, such as  information criteria and likelihood- or residual-based diagnostic measures, or out-of-sample measures like forecast accuracy at any forecast horizon. Moreover, the model class $\mathscr{M}$ can be any statistical models or black-box machine learning algorithms. In fact, as discussed in Section \ref{sec:prelim}, the term ``models'' can broadly refer to any models, learning or forecasting algorithms, and even alternatives such as policies that do not necessarily involve modeling the data. Additionally, $t$ may represent an individual time point at which  data are collected, but more generally, it may refer to the end of the $t$th time period, with each time period  encompassing multiple time points. Hence, we use the terms time period and time point interchangeably throughout this paper.

Our MPS framework draws inspiration from the Model Confidence Set (MCS) introduced by econometricians \citet{hansen2011model}, a seminal method developed for offline model selection, as well as from the recently introduced Bellman conformal inference (BCI) method in \citet{yang2024bellman}; see  an illustration in Figure \ref{fig:MPSProcedure}. However, our focus is fundamentally different from both works, as we discuss in more detail in Sections \ref{sec:prelim} and \ref{sec:method}. MPS is the first to address adaptive model (set) selection with accurate coverage guarantees in an online nonstationary environment, and the first to do so with minimal distributional assumptions and in highly flexible problem settings. Through numerical evaluation with simulated and empirical data in Section \ref{sec:numerical}, we demonstrate that MPS reliably and efficiently identifies optimal models regardless of data-generating mechanisms and forms of nonstationarity, an essential capability lacking in offline methods. Moreover, MPS frequently produces high-quality sets with small cardinality. These sets effectively identify the models that best explain the data at any given time, offering model-informed insights into the evolving dynamics.



\begin{figure}[t]
	\centering
	\includegraphics[width=0.8\textwidth]{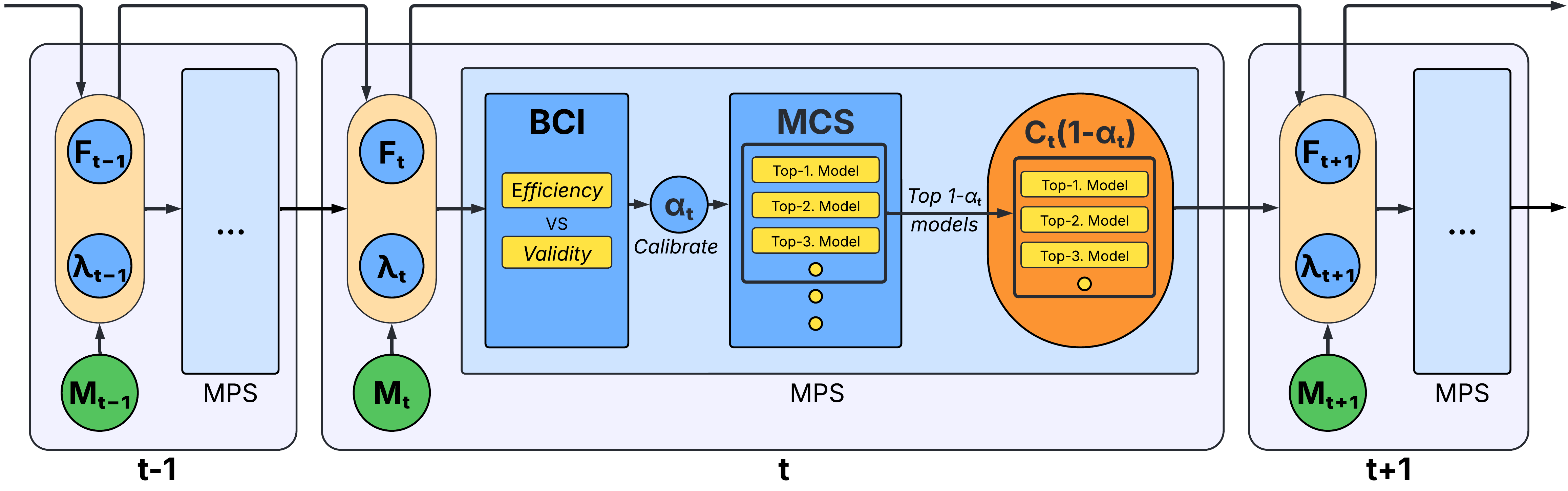}
	\caption{Illustration of the MPS procedure. More details are provided in Section \ref{sec:method}.}
	\label{fig:MPSProcedure}
\end{figure}



\paragraph{Related work} Online time series model selection methods that adapt to nonstationarity have been scarcely explored in the literature. Our work is the first to construct model \textit{prediction} sets, i.e., to infer the best models for the \textit{next} time period, in general online nonstationary settings with long-run coverage guarantees. Methodologically, MPS is closely related to MCS \citep{hansen2011model} and BCI  \citep{yang2024bellman}. However, few existing works share the same goal as ours. 

MCS have been used in traditional econometric forecasting to compare 
out-of-sample predictive performance, including applications to electricity price and economic 
growth forecasting \citep{Weron2014,Fritz2025}. 
More recently, MCS has been applied to the evaluation of machine learning--based forecasting 
models in energy, macroeconomic, and credit risk applications \citep{Huang2021,medeiros2021,Dumitrescu2022}. Moreover, there has been a rapidly growing interest in  uncertainty quantification of model selection which shares a similar objective with MCS. For example, \cite{Ferrari2015} and \cite{Zheng2019}  introduced confidence sets for model selection  via F-tests and likelihood-ratio tests, respectively, and \cite{Li2024} focused on  high-dimensional shrinkage model selection; see also the references therein as well as recent work such as \cite{Wang2025, Lewis2025, Kim2025, Cecil2025}.
Recent extensions directly built upon  \cite{hansen2011model} include \cite{arnold2024sequential}, who proposed sequential MCS by incorporating sequential testing methods, and \cite{bauer2025conditional}, who extended the MCS framework by allowing model comparison to be conducted conditional on observable 
states or regimes. However, all aforementioned developments are inherently offline; that is, the guarantees are defined for the current rather than the next time period.

At a high level, MPS  is also related to the extensive literature on conformal prediction, which was first introduced by \cite{Vovk2005} as a model-agnostic, distribution-free framework for finite-sample uncertainty quantification, with marginal coverage guarantees under exchangeability; see also \cite{Lei2018}. Recent extensions have expanded its scope to handle dependent and complex data structures, such as the work of  \cite{Tibshirani2019} on covariate shifts and \cite{Zhou2024} on random objects. On the other hand, adaptive conformal inference introduced by \cite{Gibbs2021}, as well as its variants such as  \cite{Zaffran2022, Gibbs2024} and \cite{yang2024bellman}, aims to calibrate  time-varying nominal miscoverage rate to address distribution shifts in online settings. While these works focus exclusively on predicting outcome variables, the expansion from static to dynamic settings reflects the growing need for robust uncertainty quantification in real-time systems.




\section{Preliminaries}\label{sec:prelim}

We  employ the following notation throughout the paper. For any positive integer $m$, let $[m]:=\{1,\dots,m\}$. For any set $\mathcal{S}$, we denote its cardinality by $|\mathcal{S}|$. The indicator function $\mathbf{1}(\cdot)$ takes the value one if the condition is true, and zero otherwise. The Euclidean norm of a vector is denoted by $\|\cdot\|_2$.

\paragraph{Online vs. offline model selection}  
Let $\mathscr{M}$ denote a collection of candidate models.  Offline model selection procedures \citep{McQuarrie1998, Qi2001, Castle2011}   aim to select the best model based on observed data for a fixed time period. However,  in an online setting where the data-generating process may undergo changes  as new data continue to arrive, the model  deemed  \textit{best} among $\mathscr{M}$ at any time $t$ may vary.  Moreover, in real-time applications, the goal of model selection is \textit{forward-looking} rather than retrospective. Thus, it is crucial that the model selected at any time  continues to perform reasonably well in the subsequent period. Note that in this paper, we define the \textit{best} model by any \textit{empirical} model evaluation metric, which is user-chosen and pre-determined. Under this definition, there are no tied models as this is numerically unlikely.

Obviously, the best model for the subsequent period, $\mathcal{M}_{t+1} \in\mathscr{M}$ is inherently a random object unknown at the current time $t$, and hence it is not estimable but must  be \textit{predicted}. 
The insufficiency of data under evolving dynamics, combined with the presence of comparably performing models, introduces substantial uncertainty into the prediction task. This motivates our construction of MPS in place of single-model selection.

\paragraph{Model confidence set (MCS)} The MCS procedure introduced by \citet{hansen2011model} has been highly influential in the forecasting literature. 
Departing from conventional single-model selection, it is motivated by the key fact that, in many applications, data are insufficient to identify a single model that significantly dominates all competitors.  It provides a robust mechanism for addressing the uncertainty in model selection and offers a more comprehensive view by including models that are   indistinguishable in their ability to explain the data. Specifically, given (i) an observed time series of length $n$, (ii) a collection of candidate models $\mathscr{M}$, and (iii) a model evaluation metric, the MCS procedure produces a set  $\mathcal{C}(1-\beta)\subset \mathscr{M}$ containing one or multiple best-performing models, where the cardinality of $\mathcal{C}(1-\beta)$ decreases as  $\beta\in[0,1]$ increases. Its goal is to cover the best model(s) with a given level of confidence  in an \textit{offline} setting, where the coverage targets observed time periods rather than future ones.  As a direct consequence of this different objective, there is also a conceptual nuance in the notion of the best model, which differs between \citet{hansen2011model} and our approach: They  define a \textit{population} concept of the \textit{best} model(s), $\mathcal{M}^*$, which is unobservable. Its nature is analogous to that of an unknown yet fixed parameter; so MCS, in this sense, amounts to a confidence interval for a parameter, and its offline asymptotic guarantee is established by \cite{hansen2011model}: $\liminf_{n\rightarrow\infty} \mathbb{P}(\mathcal{M}^*\not\subset \mathcal{C}(1-\beta))\leq \beta$ for any $\beta\in[0,1]$. In contrast, we define the \textit{best} model  $\mathcal{M}_{t+1}$ in the \textit{empirical} sense. It is determined by finite-sample performance and  is therefore random at time $t$ yet observable at $t+1$.

Operationally, for a given nominal miscoverage rate $\beta$, the MCS procedure takes as input an $n\times m$ \textit{loss matrix} $L=(L_{t,i})_{t\in[n], i\in[m]}$, computes a series of model equivalence tests and eliminations---typically implemented via bootstrap \citep{Bernardi2018}---and returns the set $\mathcal{C}(1-\beta)$ as output, where $n$ is the time length, and $m$ is the cardinality of $\mathscr{M}$. Each \textit{loss} $L_{t,i}$  quantifies the relative performance of model $i$ at time $t$,  which is determined based on a user-chosen model evaluation metric $\mathcal{L}$ and the available data at time $t$.
For example,  in a time series forecasting task for $Y_t\in\mathbb{R}^d$, if one uses the squared loss as $\mathcal{L}$,  then the performance of the point forecast $\widehat{Y}_{t,i}$ from model $i$ is measured by $L_{t,i}=\mathcal{L}(Y_t, \widehat{Y}_{t,i})=\|Y_t-\widehat{Y}_{t,i}\|_2^2$. Note that $\widehat{Y}_{t,i}$ may be one- or multi-step-ahead forecast for $Y_t$. In fact, $\mathcal{L}$ can be any  criterion function not necessarily tied to prediction, such as  information criterion, cross-validation, likelihood- or residual-based diagnostic measures, as long as it can be computed using a model $\mathcal{M}_i $ and a sequence $\{Y_t, Y_{t-1}, \dots\}$ as inputs, and the observation $Y_t$ may have any data structure.

Moreover, although forecasting is a leading application (which we adopt in this paper for conceptual simplicity), MCS is far more general and  is not limited to comparison of models. It can be used to select random objects, such as trading rules (see also the general discussion in \cite{hansen2011model}), since a corresponding loss matrix can be obtained as long as such random objects can be evaluated via a metric $\mathcal{L}$ and historical data, i.e., $L_{t,i}=\mathcal{L}(\mathcal{M}_i, \mathcal{F}_t)$, where $\mathcal{M}_i$ is the $i$th object (e.g., a policy), and $\mathcal{F}_t$ represents all available information up to time $t$.

%


\section{Model Prediction Set}\label{sec:method}

To address the limitations of model selection in online, nonstationary time series, we propose the Model Prediction Set (MPS).  MPS dynamically calibrates the miscoverage level $\alpha$ within the MCS framework, enabling adaptation to evolving environments while maintaining  the long-run miscoverage guarantee in \eqref{eq:target}. This guarantee is prioritized empirically in response to nonstationarity and model uncertainty, which is reflected by increased set cardinality during periods of high uncertainty as evidenced in Section \ref{sec:numerical}.

\subsection{MPS Procedure} \label{subsec:MPS}

Given any user-chosen empirical model evaluation metric $\mathcal{L}$, MPS updates a confidence set $\mathcal{C}_t$ of competitive models, based on information available up to the current time $t$, that is guaranteed to include the \textit{best} model $\mathcal{M}_{t+1}$ at time $t+1$ with a  target confidence level of  $1-\bar{\alpha}$ in the long run. It consists of two basic building blocks, which originate from MCS and BCI, respectively: 
\begin{itemize}
	\item a method for constructing model sets $\mathcal{C}_t(1-\beta)$ for any nominal miscoverage rate $\beta\in[0,1]$ based on time series data available up to time $t$; and
	\item a method  for calibrating the instantaneous nominal rates $\alpha_t$,  which adapt sequentially to evolving coverage performance in response to changing dynamics.
\end{itemize}
We illustrate the idea of  MPS  in Figure \ref{fig:MPSProcedure} and describe the methodology in this section. 

Let $\{Y_t\}$ be the data stream and, for simplicity, consider the task of forecasting $Y_t\in\mathbb{R}^d$; but as discussed,  $Y_t$ may, in general, represent data of any structure observed over time, and the application is not limited to forecasting. At each time $t$, a new observation $Y_t$ becomes available, which yields the losses $L_{t,i}$ (e.g., the forecast error $\|Y_t-\widehat{Y}_{t,i}\|_2^2$) evaluated for all candidate models $i\in\mathscr{M}$. Thus, based on the history of losses, $L_{:t}:=(L_{s,i})_{s\in[t], i\in [m]}$, given any nominal miscoverage rate $\alpha_t\in[0,1]$, we can compute $\mathcal{C}_t(1-\alpha_t)$ as the MCS obtained from the loss matrix $L_{:t}$. 

As discussed in Section \ref{sec:prelim}, MCS is an offline method and, by itself, cannot provide the online coverage guarantee in \eqref{eq:target}. Rather,  \eqref{eq:target} will be ensured by a calibration method for $\{\alpha_t\}$, which follows from the novel BCI procedure introduced by \cite{yang2024bellman}. Nonetheless, the latter is solely designed to calibrate prediction intervals for the \textit{value} of a \textit{univariate} time series in the online nonstationary setting, whereas we consider prediction sets for the best \textit{model} $\mathcal{M}_{t+1}$. Because of our different focus, our framework allows for any time series model, whether univariate, multivariate, or high-dimensional, since the loss matrix is the only essential input. Moreover, MPS is not confined to variable prediction tasks, as the criteria for defining the best model are flexible.

However, our method inherits the inner working of BCI. As a form of model predictive control \citep{Borrelli2017}, the main idea of BCI is to simulate future outcomes of the system by drawing from the observed history up to time $t$. Based on this ``historical'' simulation, an action $\alpha_t$ is planned via minimizing a cost function which simultaneously encourages efficiency and accurate control of miscoverage. Let $\beta_t=\sup\{\beta\in[0,1]:\mathcal{M}_{t+1}\in\mathcal{C}_t(1-\beta)\}$ with marginal distribution $F_t$. 
The calibration of $\alpha_t$ is given as below:
\begin{align}	\label{eq:BCI}
	\alpha_t^*=\underset{\alpha}{\min}\, \mathbb{E}_{\beta_t\sim F_t}\left\{\lvert \mathcal{C}_t(1-\alpha)\rvert + \lambda_t\max[\mathbf{1}(\alpha>\beta_t)-\bar{\alpha}, 0]\right\}
\end{align}
and 
\[
\alpha_t=\alpha_t^*\mathbf{1}(\lambda_t<\lambda_{\max}),
\]
where $\lambda_{\max}>0$ is a pre-specified threshold. Since $\lvert \mathcal{C}_t(1-\alpha)\rvert$ is non-increasing in $\alpha$ while $\mathbf{1}(\alpha>\beta_t)$ is non-decreasing in $\alpha$, the trade-off between efficiency (i.e., the cardinality of the model set) and validity (i.e., control of the miscoverage rate) is balanced by $\lambda_t$, which is a relative weight to penalize miscoverage at time $t$. In addition,  $\lambda_t$ is adaptively updated by $\lambda_{t+1}=\lambda_t+\gamma[\mathbf{1}(\alpha>\beta_t)-\bar{\alpha}]$, where $\gamma>0$ is the step size. This update rule ensures  that a miscoverage at time $t-1$ leads to a larger $\lambda_t$, up to a maximum threshold $\lambda_{\max}$, and is the key lever to  achieve \eqref{eq:target}.  Since $\alpha\in[0,1]$, the optimization in \eqref{eq:BCI} can be easily implemented via  one-dimensional grid search. 

In practice, we use the empirical distribution of $\beta_{t-1}, \dots, \beta_{t-\tau}$ to approximate $F_t$, where $\tau$ is a fixed block size, and a reasonable range is $\tau\in [100, 500]$. Due to the approximation of $F_t$, given an initial training dataset $\{Y_1,\dots, Y_n\}$,  the MPS algorithm requires an offline initialization of $\beta_{n-1}, \dots, \beta_{n-\tau}$,  before starting online updates after time $t=n$. This initialization uses MCS without any calibration; see lines 2--9 in Algorithm \ref{alg:MPS}. We recommend setting $\lambda_{\max}=2000$ and $\gamma=c\lambda_{\max}$ with $c=0.2$, which are used throughout our numerical studies. The bootstrap sample size for implementing MCS is set to $B=100$, and the grid for searching the nominal miscoverage rate is set to $\mathcal{G}=\{k/20: k=0,1, \dots, 19\}$. The detailed implementation is given in Algorithm \ref{alg:MPS}.



\begin{algorithm}[t]
	\caption{Model Prediction Set Algorithm}
	\label{alg:MPS}
	\textbf{Input:}  A collection of candidate models $\mathscr{M}$ indexed by $\{1,\dots, m\}$, target miscoverage level $\bar{\alpha}$, threshold $\lambda_{\max}$, relative step size $c$, step size $\gamma=c\lambda_{\max}$, block size $\tau$,  initial training data $\{Y_1, \dots, Y_{n}\}$ with $n\geq \tau$, model evaluation metric $\mathcal{L}$,  bootstrap sample size $B$, grid $\mathcal{G}$\\
	\textbf{Offline initialization of $\{\beta_{t},  t= n-\tau+1, \dots, n-1\}$:}\\
	\hspace*{5mm}\textbf{for} $t\in\{n-\tau+2, \dots, n\}$ \textbf{do}\\
	\hspace*{10mm}Obtain the \textit{best} model $\mathcal{M}_t=\argmin_{i\in \mathscr{M}}L_{t,i}$\\ 
	\hspace*{10mm}\textbf{for} $\alpha\in \mathcal{G}$ \textbf{do}\\
	\hspace*{15mm} $\mathcal{C}_{t-1}(1-\alpha)=\text{MCS}(L_{:t-1}, \alpha, B)$\\
	\hspace{10mm}\textbf{end for}\\
	\hspace{10mm}  $\beta_{t-1}= \argmax\{\beta\in\mathcal{G}: \mathcal{M}_t\in \mathcal{C}_{t-1}(1-\beta)\}$\\
	\hspace*{5mm}\textbf{end for}\\
	\textbf{Initialize:} $\lambda_{n}= \lambda_{\max}/2$, $\alpha_n=\bar{\alpha}$, $\mathcal{C}_n(1-\alpha)=\text{MCS}(L_{:n}, \alpha, B)$ for all $\alpha\in\mathcal{G}$\\
	\textbf{repeat at each time step $t \geq n+1$} \\
	\hspace*{5mm}Observe new data $Y_t$ and compute losses $L_{t,i}$ for all $i\in\mathscr{M}$\\
	\hspace*{5mm}Obtain the \textit{best} model $\mathcal{M}_t=\argmin_{i\in \mathscr{M}}L_{t,i}$  \\
	\hspace*{5mm}$\beta_{t-1}= \argmax\{\beta\in\mathcal{G}: \mathcal{M}_t\in \mathcal{C}_{t-1}(1-\beta)\}$ \\
	\hspace*{5mm}\textbf{for} $\alpha\in \mathcal{G}$ \textbf{do}\\
	\hspace*{10mm}Update $\mathcal{C}_{t}(1-\alpha)=\text{MCS}(L_{:t}, \alpha, B)$\\
	\hspace{5mm}\textbf{end for}\\
	\hspace*{5mm}Update $\lambda_t= \lambda_{t-1}+ \gamma[\mathbf{1}(\alpha_{t-1}>\beta_{t-1})-\bar{\alpha}] $\\
	\hspace*{5mm}$\alpha_t^*=\argmin_{\alpha\in\mathcal{G}} \frac{1}{\tau}\sum_{s=t-1}^{t-\tau}\{|\mathcal{C}_t(1-\alpha)| +\lambda_t \max [\mathbf{1}(\alpha>\beta_s)-\bar{\alpha}, 0] \}$\\  
	\hspace*{5mm}$\alpha_t=\alpha_t^*\mathbf{1}(\lambda_t<\lambda_{\max})$\\
	\hspace*{5mm}\textbf{Output:} Calibrated model prediction set $\mathcal{C}_t(1-\alpha_t)$
\end{algorithm}

\subsection{Nonasymptotic Coverage Guarantee}\label{subsec:thm}

The following theorem, which requires no distributional assumption on $F_t$, immediately leads to the desired coverage guarantee in \eqref{eq:target} for MPS.

\begin{theorem}\label{thm}
	If $\gamma=c\lambda_{\max}$ for some constant $c\in(0,1)$, then for any nonnegative integer $n$, $\lvert T^{-1}\sum_{t=n+1}^{n+T} \mathbf{1}(\alpha_t>\beta_t)-\bar{\alpha}\rvert\leq (c+1)/(cT)$.
\end{theorem}
\begin{proof}
	The proof of this theorem directly follows from that of Theorem 1 in \cite{yang2024bellman}, once the following  assumptions are verified: for any $t\geq 1$, the prediction set $\mathcal{C}_t(1-\beta)\subset \mathscr{M}$ satisfies (i) \textit{monotonicity}: $\mathcal{C}_t(1-\beta_1) \subset \mathcal{C}_t(1-\beta_2)$ if $\beta_1>\beta_2$, and (ii) \textit{safeguard}: $\mathcal{C}_t(1)=\mathscr{M}$, i.e. $\mathbb{P}(\mathcal{M}_{t+1}\in \mathcal{C}_t(1))=1$. By \cite{hansen2011model}, it is clear that $\mathcal{C}_t(1-\beta)$ constructed from MCS for any $\beta\in[0,1]$ and $t\geq1$ satisfies both conditions.
\end{proof}

\begin{remark}
	MPS builds upon MCS, but Theorem \ref{thm} does not rely on the offline asymptotic coverage property of  the MCS procedure established  by \cite{hansen2011model}, as it is an online coverage guarantee. Thus, assumptions made in \cite{hansen2011model} to ensure the offline asymptotic coverage of MCS with regard to the population concept $\mathcal{M}^*$ (i.e., the ``true'' set of best models) are irrelevant and not needed. However, if such assumptions hold, a nice by-product is that,  if $\mathcal{M}_{t+1}\in\mathcal{M}^*$, then MPS also has the asymptotic coverage property:  $\liminf_{t\rightarrow\infty} \mathbb{P}(\mathcal{M}_{t+1}\notin \mathcal{C}_t(1-\alpha_t))\leq \alpha_t$.
\end{remark}

\section{Numerical Evaluation}\label{sec:numerical}

Our numerical studies demonstrate that MPS maintains a well-controlled miscoverage rate despite   changes in the data-generating process and model performance over time.  As shown, during periods of heightened uncertainty caused by nonstationarity and model ambiguity, MPS prioritizes maintaining the target coverage  by adaptively increasing cardinality, thereby acknowledging model uncertainty and data limitations. By contrast, the offline MCS  lacks adaptability to evolving dynamics and, more seriously, often produces trivial sets that include all candidate models. 

Interestingly, as MPS adaptvely balances efficiency and accurate coverage, during periods when achieving good coverage is relatively easy, it tends to produce highly precise prediction sets with extremely small cardinalities. We refer to these updates as \textit{quality sets}, which occur frequently throughout the MPS update process. These sets effectively identify models that best explain the data at any given time and offer valuable insights into gradual or abrupt changes in the data-generating mechanism.




\subsection{Simulation Experiments}\label{subsec:sim}

\paragraph{Experiment with designed loss matrices}
The MPS procedure can be applied to any model class, data-generating process, or model evaluation metric. Its essential input is the loss $L_{t,i}$, which measures the relative performance of model $i$ in explaining the data at time $t$, based on a specific evaluation metric $\mathcal{L}$. In view of this generality, we first conduct an experiment based on randomly generated loss matrices, which allows us to  control the comparative performance of different candidates $i \in \mathscr{M}$ over time.

We generate the loss matrix $L=(L_{t,i})_{t\in[T], i\in[m]}$ under three different designs, where $T=2000$ is the total time length, and $m=10$ represents the number of candidates in $\mathscr{M}$:
\begin{itemize}
	\item[(a)] \textit{All candidates  perform similarly over the entire period:} All entries in the loss matrix are generated from the uniform distribution $U(0,2)$. 
	\item[(b)] \textit{Two candidates exhibit recurring local changes, while the others maintain similar  performance over the entire period:} Two columns of $L$ contain continuous blocks of 25 smaller entries, generated from $U(0.5,1.5)$ within every 50 entries, while  the remaining entries in these columns are drawn from $U(1,2)$. All entries of the remaining eight columns are generated from $U(0,2)$. 
	\item[(c)] \textit{Two candidates exhibit gradual changes with a common turning point, while the others maintain similar performance  over the entire period:} All entries in one column are generated as $U(\mu_t,1)$, where $\mu_t=\frac{2t}{T}\mathbf{1}(0\leq t\leq T/2)+\frac{2(T-t)}{T}\mathbf{1}(T/2< t\leq T)$ increases from 0.5 to 1.5 at time $t=1000$, and then decreases back to 0.5. All entries in another column are generated as $U(\mu_t^\prime,1)$, where $\mu_t^\prime = \frac{T-2t}{T}\mathbf{1}(0\leq t\leq T/2)+\frac{2t-T}{T}\mathbf{1}(T/2< t\leq T)$ decreases from 1.5 to 0.5 at time $t=1000$, and then increases back to 1.5. All entries of the remaining eight columns are drawn from $U(0,2)$.
\end{itemize}

\begin{figure}[t]
	\centering
	\includegraphics[width = 0.8\textwidth]{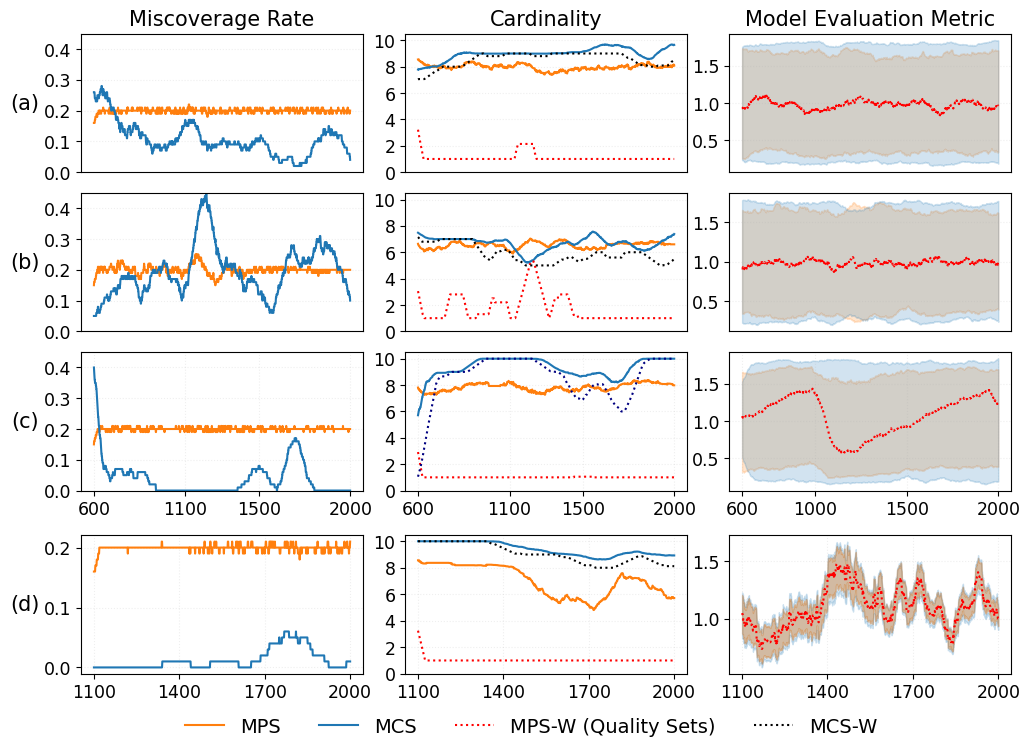}
	\caption{Miscoverage rate, cardinality, and the range of losses (i.e., the values of the model evaluation metric) for the models selected by MPS and MCS for simulation experiments. Results are shown for three loss matrix designs (a)--(c), as well as for the model fitting experiment in  panel (d).}
	\label{fig:sim}
\end{figure}

We run Algorithm \ref{alg:MPS} with  $\tau=100$, $\bar{\alpha}=0.2$, and $n=500$, so the remaining $T-n=1500$ time points are used to demonstrate the MPS updates. We benchmark MPS against the offline MCS procedure and display the results under the three designs in panels (a)--(c) of Figure \ref{fig:sim}. 

\paragraph{Results}
The left panel of Figure \ref{fig:sim} shows the miscoverage rate for the models selected by MPS and MCS, computed using a moving window of size 100. It can be seen that MPS consistently maintains the miscoverage rate close to the nominal level of 0.2 across  designs (a)--(c). By contrast, MCS fails to control the  miscoverage rate under all three designs. In particular, for design (c), MCS often leads to zero miscoverage due to its inclusion of all candidates, which is  uninformative and lacks adaptability to evolving comparative model performance.

The middle panel of Figure \ref{fig:sim} shows the cardinality of the model sets selected by MPS and MCS. The solid curves represent the average cardinality over a moving window of size 100, consistent with the calculation of miscoverage rates. The moving average cardinality confirms the tendency of MCS to trivially select all candidates under design (c), whereas MPS avoids this issue. However, it is worth noting that averaging the cardinality over a moving window does not reveal the full picture, as the cardinality can vary at each time step. Thus, we additionally report the minimum cardinality over a moving window of size 20 for both MPS and MCS, labeled as MPS-W (Quality Sets) and MCS-W, respectively, where W stands for ``windowed'' as it essentially corresponds to a windowed procedure: at each $t$, the minimal cardinality set from the last 20 steps is adopted.
Based on this measure, we observe that MPS frequently produces sets with extremely small cardinality---henceforth referred to as \textit{quality sets}. We also observe notable spikes in the cardinality of the MPS quality sets under design (b), as smaller losses occur for more candidates during those periods. This highlights MPS's sensitivity to local patterns in the update process. In contrast, MCS-W closely resembles MCS, indicating that MCS never produces  low-cardinality sets throughout the entire period. This reveals that MCS lacks the efficiency demonstrated by MPS in the online setting.

The right panel of Figure \ref{fig:sim} visualizes the range of losses for the candidates selected by MPS and MCS, where the shaded areas span the maximum and minimum losses among the selected candidates. Note that they correspond to values of the model evaluation metric in real applications and  are thus of practical interest. Additionally, we display the average loss of the MPS quality sets. All reported values are averaged over a moving window of size 100 to smooth out the patterns. 
We observe that the loss range of MPS tends to fall within that of MCS, indicating greater overall stability. Moreover, under design (c), the quality sets display significant fluctuations in loss values. This reflects MPS's sensitivity to the turning point at $t = 1000$: once enough data accumulate, the quality updates are able to select the best-performing model after the shift. In contrast, MCS shows almost no response.

\paragraph{Experiment with model fitting}
An experiment involving actual model fitting is further conducted: We generate a time series $\{Y_t\}_{t=1}^T$ with  $T=2000$ from  $Y_t=0.3 Y_{t-1}+\varepsilon_t+0.3\mathbf{1}(1\leq t\leq 1000)\varepsilon_{t-1}$, where $\varepsilon_t\overset{i.i.d}{\sim}N(0,1)$, and run Algorithm \ref{alg:MPS} with $n=1000$ and $\tau=500$. Here $\mathscr{M}$ consists of AR$(p)$ and MA$(q)$ models with $1\leq p,q\leq 5$, and $\mathcal{L}$ is the squared one-step-ahead forecast error (FE). 

\paragraph{Results}
Panel (d) of Figure \ref{fig:sim}  presents results obtained using the same procedure as in the previous experiment. Similar to the findings from (c),  MPS maintains accurate control of the miscoverage rate and  yield quality sets with much smaller cardinality (mostly one) than MCS, exhibiting exceptional stability after an initial adaptation period. While fluctuations in the FE are present, possibly due to all candidate models being misspecified, MPS exhibits a narrower loss range compared to MCS.

\subsection{Empirical Analysis}\label{sec:data}

\paragraph{Data and settings} We consider two real-world time series: the daily average oil temperature (OT) with total time length $T=726$, computed from hourly data in the Electricity Transformer Temperature (ETT) dataset \citep{Zhou2021}, and the daily CBOE Volatility Index (VIX) from 2020-03-25 to 2025-03-25 which, after differencing, yields $T = 1304$. For simplicity, we focus on univariate forecasting for each data, and adopt the squared one-step-ahead FE as the evaluation metric $\mathcal{L}$. For both data, $\mathscr{M}$ includes 10 candidate models, and we set $n=240$, $\tau=150$, and $\bar{\alpha}=0.2$. For OT, $\mathscr{M}$ contains an AR(1) model, and AR(1) models coupled with nine different combinations of polynomial (linear, quadratic, or cubic) and seasonal (one, two, or three harmonics) trends, using a seasonal period of seven days. For VIX,  $\mathscr{M}$ contains AR($p$) models with $1\leq p\leq 4$, threshold AR \citep{Tong2012}, smooth transition AR \citep{Teraesvirta1994}, and machine learning methods \citep{James2021}:  random forest, vanilla neural network, Long Short-Term Memory (LSTM), and the Transformer \citep{Vaswani2017}. See Appendix \ref{sec:a1} for more details on the data and models.

\paragraph{Results} 
Figure \ref{fig:ETT1} presents results using the method from Section \ref{subsec:sim}. For both data,   MPS maintains the miscoverage rate close to the nominal level of 0.2, whereas MCS  trivially selects all 10 candidate models throughout the latter part of the OT data and  the entire update period of the VIX data. Notably, MPS consistently produces quality sets with near-unity cardinality for both datasets, while MCS always yield large or full sets; in fact, for VIX, the results of MCS-W are identical to those of MCS. Combining the miscoverage rate and cardinality results, we confirm the finding  from previous experiments: MPS offers greater robustness, adaptivity, and efficiency in online settings. Additionally, we  observe that the range of forecast errors from MPS tends to fall within that of MCS, indicating greater overall stability. Moreover, a closer look at the quality sets reveals interesting transitions in the favored models: For OT, MPS dynamically adapts its selection from an AR(1) model with a quadratic trend to one with a cubic trend (both with one harmonic), before eventually converging to the pure AR(1) model. For VIX, MPS initially favors LSTM, gradually shifts to AR(1), and then quickly converges to the Transformer for most of the period; see the right panel in Figure \ref{fig:ETT1}.

\begin{figure}[t]
	\centering
	\includegraphics[width=0.9\textwidth]{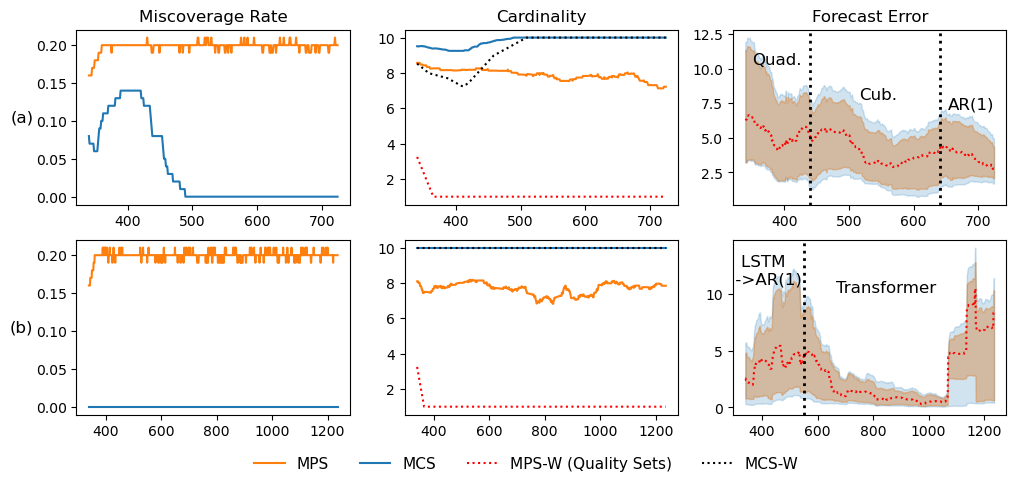}
	\caption{Comparison of MPS and MCS performance on real-world data: (a) OT and (b) VIX.}
	\label{fig:ETT1}
\end{figure}

\section{Conclusion and Discussion}\label{sec:conclude}

We introduced the Model Prediction Set (MPS), a novel framework to tackle the challenge of adaptive model selection in online nonstationary settings. MPS is the first to address adaptive model (set) selection with long-run coverage guarantees in an online nonstationary environment, and the first to do so with minimal distributional assumptions and in highly flexible problem settings. Numerical studies demonstrated its practical advantages over offline methods in terms of robustness, efficiency, and adaptivity. In particular,  MPS was found to frequently produce quality sets that sensitively adapt to changing dynamics and model performance over time. Since it operates on loss values, MPS imposes no constraints on the data-generating mechanism, data structure, model class, training method, and evaluation metric. Its utility extends beyond forecasting; e.g., it may be applied to select random objects or policies in nonstationary environments \citep{Sutton2018}.

There are also limitations that need to be addressed in future research. Our numerical studies  considered at most $m=10$ candidate models due to the dramatic increase in computation time as $m$ grows. The computational bottleneck lies in the bootstrap procedure used by MCS. Each update step in our numerical experiments took an average of 25.6 seconds, running on an AMD Epyc server with 128 CPU cores and 492 GB of RAM. However, it is  noteworthy that MCS is used only to produce a preliminary model set. The MPS framework remains valid when combined with other model set construction methods, as its nonasymptotic coverage guarantee is ensured by the calibration procedure rather than by MCS. Therefore, advances in computational efficiency for MCS or alternative methods for model set construction would further enhance the scalability of MPS.

\bibliography{MPS}

\newpage
\appendix

\section{Additional Details for the Empirical Analysis}\label{sec:a1}

This appendix provides additional details about the empirical analysis in Section \ref{sec:data} and about the experiment illustrated in Figure \ref{fig:ModelSelction}.

\subsection{Data}
\paragraph{ETT data} The Electricity Transformer Temperature (ETT) dataset is a well-established benchmark dataset in time series forecasting \citep{Zhou2021}. Our analysis focuses specifically on the oil temperature (OT) from the ETTh1 (ETT-hourly-1) subset, which comprises hourly measurements of critical operational parameters from a 220kV power transformer in China, recorded over a two-year period (July 2016--July 2018). The dataset includes seven key variables: OT as the target measurement and six complementary power load features. 
We use only the hourly OT data (without other features) and compute the daily average to obtain the daily OT series. No further transformation is applied.

\paragraph{CBOE Volatility Index}
The CBOE Volatility Index (VIX), known as the market's ``fear gauge,'' measures 30-day expected stock market volatility derived from S\&P 500 index options \citep{FRBLn.d.}. Maintained by the Chicago Board Options Exchange (CBOE), the VIX reflects investor sentiment and risk expectations in real time. As a forward-looking indicator, it serves as a benchmark for volatility trading and risk management.We use daily VIX data from 2020-03-25 to 2025-03-25. Since financial data typically exhibit stochastic rather than deterministic trends (e.g., random walk behavior), we difference the series prior to training.

Figure \ref{fig:tsplot} shows the time series plots of the daily average OT and daily VIX data (before differencing). 

\begin{figure}[h]
	\centering
	\includegraphics[width=\textwidth]{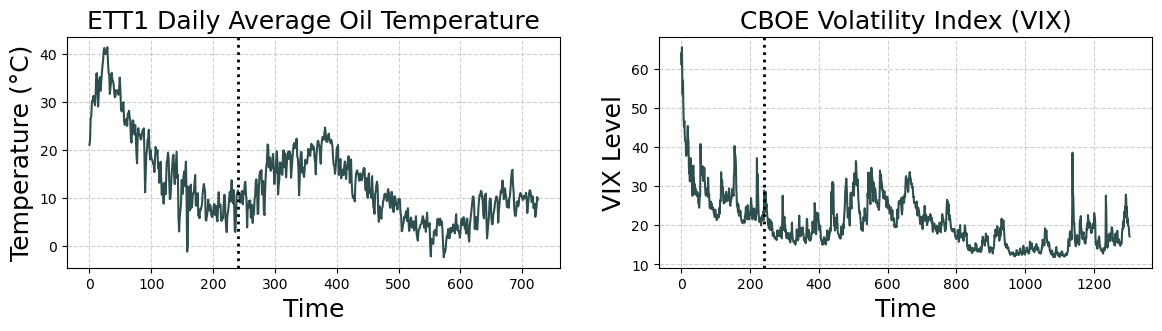}
	\caption{Time series plots of (a) OT and (b) VIX. The dashed lines indicate the end of the initial training set at $n=240$.}
	\label{fig:tsplot}
\end{figure}

\subsection{Models and Training Details}

We consider 10 candidate models for each data:
\begin{itemize}
	\item For OT, these include  an AR(1) model, and AR(1) models coupled with nine different combinations of polynomial (linear, quadratic, or cubic) and seasonal (one, two, or three harmonics) trends, using a seasonal period of seven days.
	\item For VIX, these include AR($p$) models with $1\leq p\leq 4$, threshold AR (TAR), smooth transition AR (STAR), random forest, vanilla neural network (VNN), Long Short-Term Memory (LSTM), and the Transformer. 
\end{itemize}

The statistical models (AR, AR with polynomial and seasonal trends, TAR, and STAR) are implemented in R. The AR($p$) model is given by $Y_t = \phi_0+ \sum_{i=1}^{p}\phi_i Y_{t-i} + \varepsilon_t$, for $p\geq1$. An AR(1) model with  polynomial (linear, quadratic, or cubic) and seasonal (one, two, or three harmonics) trends is $Y_t = \phi_0+\phi_1Y_{t-1} + m_t+S_t+ \varepsilon_t$
where $m_t=\sum_{i=1}^q\gamma_i t^i$,
$S_t = \sum_{j=1}^r [\alpha_j\sin(2\pi t/s)+\beta_j\cos(2\pi t/s)]$, and $s$ is the seasonal period, for $1\leq q, r\leq 3$. These models are estimated via conditional least squares using the \texttt{arima} function from the \texttt{stats} package. The two-regime TAR model 
\[
Y_t =
\begin{cases}
	\phi_{1,0}	+\phi_{1,1} Y_{t-1} + \varepsilon_t, & \text{if } Y_{t-2} \leq r \\
	\phi_{2,0}	+ \phi_{2,1} Y_{t-1} + \varepsilon_t, & \text{if } Y_{t-2} > r
\end{cases}
\]
is fitted by the minimizing AIC method using the \texttt{tar} function from the \texttt{TSA} package. The STAR model 
$Y_t = \sum_{j=1}^{K}(\phi_{j,0} +\phi_{j,1} Y_{t-1}) \cdot G_j(Y_{t-2}; \gamma_j, c_j) + \varepsilon_t$, where 
$G_j(Y_{t-2}; \gamma_j, c_j) = \frac{\exp(-\gamma_j (Y_{t-2} - c_j))}{\sum_{k=1}^{K} \exp(-\gamma_k (Y_{t-2} - c_k))}$ is fitted via nonlinear least squares estimation using the \texttt{star} function from the \texttt{tsDyn} package, which automatically selects $K$ from $\{1,\dots, K_{\max}\}$, and we set $K_{\max}=5$.

The machine learning methods are all implemented in Python using the past 10 lags as input features: random forest via \texttt{sklearn.ensemble}, and VNN, LSTM, and Transformer models via \texttt{tensorflow.keras}; see more details in Table \ref{tab:model_params}.

\begin{table}[h]
	\centering
	\caption{Specification for random forest, VNN, LSTM, and Transformer.}
	\label{tab:model_params}
	\begin{tabular}{lp{10cm}}
		\toprule
		\textbf{Model} & \textbf{Key Specifications}\\
		\midrule
		Random Forest & 10 lags as input features, 100 trees, MSE splitting \\
		VNN & 10 lags as input features, 2 hidden layers (10 units each, ReLU), linear output layer, Adam, batch size = 16, epochs = 10\\
		LSTM & 10 lags as sequence input, 100 units (tanh for the cell state and sigmoid for the gates), linear output layer,  Adam, batch size = 16, epochs = 10\\
		Transformer & 2-layer decoder-only Transformer, 8-head self-attention, positional encoding, feedforward dimension = 16, ReLU activation, final linear output layer, Adam, batch size = 16, epochs = 10\\
		\bottomrule
	\end{tabular}
	\label{tab:modelpara}
\end{table}

Based on the computed losses from model fitting, the MPS procedure is implemented via the R package \texttt{MCS}.

\subsection{Experimental Settings for Figure \ref{fig:ModelSelction}}

Figure \ref{fig:ModelSelction} is generated under the same experimental setting as the empirical analysis of OT in Section \ref{sec:data}, with the addition of results from single-model selection methods (AIC, BIC, and cross-validation). The cross-validation (CV) method refers to time-series CV, where at each time $t$,  we split the data up to time $t$ into a fitting portion (the first 90\%) and a hold-out  portion (the last 10\%). The former is used for model training, and the latter for evaluating forecast performance via a rolling one-step-ahead forecasting procedure.

\end{document}